\documentclass[lettersize,journal]{IEEEtran}
\usepackage{amsmath,amsfonts}
\usepackage{algorithmic}
\usepackage{algorithm}
\usepackage{array}
\usepackage{textcomp}
\usepackage{stfloats}
\usepackage{url}
\usepackage{verbatim}
\usepackage{graphicx}
\usepackage{cite}

\usepackage{amsthm}
\usepackage[table]{xcolor}
\usepackage{mdframed}

\newtheorem{lemma}{Lemma}
\newtheorem{theorem}{Theorem}

\usepackage[labelformat=simple]{subcaption}

\usepackage{multirow}
\usepackage{booktabs} 
\usepackage{amssymb}
\begin{document}

\title{Enhancing Quantization-Aware Training on Edge Devices via Relative Entropy Coreset Selection and Cascaded Layer Correction}

\author{Yujia Tong, Jingling Yuan,~\IEEEmembership{Senior Member,~IEEE}, Chuang Hu,~\IEEEmembership{Member,~IEEE}
    \IEEEcompsocitemizethanks{
     \IEEEcompsocthanksitem  Yujia Tong, Jingling Yuan are with Hubei Key Laboratory of Trasportation Internet of Things, School of Computer Science and Artificial Intelligence, Wuhan University of Technology, Hubei 430072, China. (E-mail: \{tyjjjj, yjl\}@whut.edu.cn.)
    \IEEEcompsocthanksitem  Chuang Hu is with the School of Computer Science, Wuhan University, Hubei 430072, China. (E-mail: handc@whu.edu.cn.)
    }
}

\maketitle

\begin{abstract}
With the development of mobile and edge computing, the demand for low-bit quantized models on edge devices is increasing to achieve efficient deployment. To enhance the performance, it is often necessary to retrain the quantized models using edge data. However, due to privacy concerns, certain sensitive data can only be processed on edge devices. Therefore, employing Quantization-Aware Training (QAT) on edge devices has become an effective solution. Nevertheless, traditional QAT relies on the complete dataset for training, which incurs a huge computational cost.  Coreset selection techniques can mitigate this issue by training on the most representative subsets. However, existing methods struggle to eliminate quantization errors in the model when using small-scale datasets (e.g., only 10\% of the data), leading to significant performance degradation. To address these issues, we propose QuaRC, a QAT framework with coresets on edge devices, which consists of two main phases: In the coreset selection phase, QuaRC introduces the “Relative Entropy Score” to identify the subsets that most effectively capture the model’s quantization errors. During the training phase, QuaRC employs the Cascaded Layer Correction strategy to align the intermediate layer outputs of the quantized model with those of the full-precision model, thereby effectively reducing the quantization errors in the intermediate layers. Experimental results demonstrate the effectiveness of our approach. For instance, when quantizing ResNet-18 to 2-bit using a 1\% data subset, QuaRC achieves a 5.72\% improvement in Top-1 accuracy on the ImageNet-1K dataset compared to state-of-the-art techniques.

\end{abstract}


\section{Introduction}

\IEEEPARstart{W}{ith} the rapid growth of mobile devices and edge computing, real‐time computer vision tasks—such as scene classification\cite{cheng2017remote} and crack detection\cite{zakeri2017image}—are increasingly executed on smartphones, laptops, and Unmanned Aerial Vehicles (UAVs). These applications must run under stringent compute and battery limits \cite{shuvo2022efficient, zhao2022survey}. Although lightweight architectures like MobileNetV2 \cite{howard2017mobilenets} reduce overhead, they still struggle to meet real‐time latency targets on image data.

Model quantization \cite{chen2022energy, luo2025bi} addresses this gap by lowering the bitwidth of weights and activations, cutting both storage and arithmetic cost. Quantization splits into Post‐Training Quantization (PTQ) \cite{fang2020post, wang2020towards} and Quantization‐Aware Training (QAT) \cite{zhou2016dorefa, esser2020learned}. PTQ applies quantization to a trained model without retraining, but accuracy degrades sharply at low bitwidths ($\leq$4 bits). QAT, by simulating quantization effects during training, adapts the model to quantization noise and retains higher accuracy. As a result, QAT is the preferred route for aggressive bitwidth reduction on edge devices.

Conventional QAT pipelines, however, face many limitations in practical applications. To obtain a quantized model, traditional QAT usually involves transmitting data generated on the edge devices to the cloud, where the model is trained with quantization using this data. The quantized model is then transmitted back to the edge devices. However, due to data privacy and communication latency issues, data generated on edge devices often needs to be processed locally. Therefore, QAT needs to be performed directly on the edge devices. Traditional QAT methods require training with the entire dataset, which incurs significant computational and time overheads. This is unacceptable for edge devices. As a result, improving the efficiency of QAT on edge devices has become a critical research direction. 

A promising strategy to mitigate the computational overhead of QAT is the utilization of \textit{coreset} techniques ~\cite{welling2009herding,paul2021deep,killamsetty2021grad,sorscher2022beyond}. These methods identify a representative subset of the dataset, creating a smaller yet statistically and structurally representative training set that closely approximates the original data. By applying coreset selection, the training process can be substantially accelerated, reducing computational costs while preserving the model's performance and robustness. 

Applying coresets to QAT poses two main challenges. First, \textit{how to select samples that accurately capture quantization errors?} Quantization errors—introduced by reduced precision—directly drive accuracy loss. Existing coreset metrics (e.g., error vector score, disagreement score \cite{huangrobust}) do not explicitly measure a sample’s contribution to quantization errors. Precise selection criteria are therefore needed to ensure the coreset highlights the most error‐sensitive inputs.

Second, \textit{how to retrain on a small coreset so as to minimize quantization error propagation?} Standard QAT losses, possibly augmented with knowledge distillation, work well on full datasets but falter on small subsets. Limited samples fail to correct accumulated errors in intermediate layers, leading to degraded performance. A tailored training strategy is required to suppress layer‐wise quantization errors when data are scarce.

In this paper, we address two key challenges in coreset selection for QAT, aiming to mitigate the performance degradation of quantized models when using a small-scale coreset. We propose QuaRC, a 
\textit{\textbf{Qua}}ntization-aware training framework on edge devices via \textit{\textbf{R}}elative entropy coreset selection and \textbf{\textit{C}}ascaded layer correction. Specifically, during the coreset selection phase, we first input the same samples into both the full-precision model and the quantized model, and calculate the difference in their outputs using relative entropy. We find that samples with higher relative entropy are more capable of capturing quantization errors, and training with such samples can improve model accuracy. To verify this phenomenon, we test the impact of samples with different relative entropies on the accuracy of quantized models and calculate the Spearman correlation coefficient between relative entropy and model performance. The results show a strong correlation between the two. Based on this observation, we propose the \textbf{Relative Entropy Score (RES)} as a criterion for coreset selection. In the quantized model training phase, to address the accumulation of quantization errors across layers, we introduce a \textbf{Cascaded Layer Correction (CLC)} strategy. This approach aligns the intermediate layer outputs of the quantized model with those of the full-precision model, thereby reducing errors at intermediate layers. We conduct extensive experiments to validate our method. Compared to previous methods, our approach can significantly enhance the performance of the quantized model when selecting a small-scale coreset (\(10\%\) or less) for QAT. 

In summary, our contributions are:
\begin{itemize}
   \item We propose a new metric for coreset selection—Relative Entropy Score (RES). This metric can be used to select the samples that accurately capture quantization errors.
   \item We propose the Cascaded Layer Correction (CLC) training strategy that can effectively eliminate the quantization errors in the intermediate layers of the model.
   \item We conduct extensive experiments to validate our method. Compared to previous methods, our approach can significantly enhance the performance of the quantized model when training with a small-scale coreset (\(10\%\) or less). 
   \item To further demonstrate the effectiveness of QuaRC in the practical scenario, we conduct a case study on an unmanned aerial vehicle (UAV) for concrete crack
detection.

\end{itemize}

\section{Background and Motivation}
In this section, we first introduce the background of QAT. Subsequently,  we discuss the challenges of existing coreset selection methods when applied to QAT, which motivates our proposed method.

\subsection{Quantization-Aware Training}
To simulate the rounding and clamping errors of the quantized model during the inference process, QAT \cite{li2022q,nagel2022overcoming} introduces fake quantization nodes that perform quantization and dequantization operations on floating-point numbers. For \( n \)-bit signed quantization, given a scaling factor \( s \), the full-precision number \( x^r \) is converted to \( x^q \) at the fake quantization node through quantization and dequantization operations:
\begin{equation}
\label{eq:quantization}
x^q = q(x^r) = s \times \left\lfloor \text{clamp}\left(\frac{x^r}{s}, -Q_N, Q_P\right) \right\rceil
\end{equation}
where \( \left\lfloor \cdot \right\rceil \) is an operation used to round its input to the nearest integer. The function \(  clamp(x, r_{\text{low}}, r_{high}) \)  returns \( x \) while ensuring that values below \( r_{\text{low}} \) are set to \( r_{\text{low}} \), and values above \( r_{\text{high}} \) are set to \( r_{\text{high}} \).

Typically, in a neural network where both activations and weights are quantized, the forward and backward propagation processes can be simply represented as:

\begin{align}
\label{eq:process}
    &\text{Forward:} \quad \text{Output}(x) = x^q \cdot w^q = q(x^r) \cdot q(w^r)\notag\\
    &\text{Backward:} \quad \frac{\partial \mathcal{J}}{\partial x^r} = \left\{
    \begin{array}{ll}
    \frac{\partial \mathcal{J}}{\partial x^q} & \text{if } x \in [-Q_N^x, Q_P^x] \\
    0 & \text{otherwise}
    \end{array}
    \right. \\
    &\quad\quad\quad\quad\quad \frac{\partial \mathcal{J}}{\partial {\bf w^r}} = \left\{
    \begin{array}{ll}
    \frac{\partial \mathcal{J}}{\partial {\bf w^q}} & \text{if } {\bf w} \in [-Q_N^{\bf w}, Q_P^{\bf w}] \\
    0 & \text{otherwise}\notag
    \end{array}
    \right.
\end{align}
where \( \mathcal{J} \) is the loss function, \( q(\cdot) \)  is employed during forward propagation, and the straight-through estimator (STE) \cite{bengio2013estimating} is utilized to maintain the derivation of the gradient during backward propagation.

In image classification tasks\cite{chen2021crossvit}, neural networks primarily extract features and enable effective classification. We can regard this process as a transformation from input data to output results. 
Specifically, the workflow entails mapping the input image \( x \) to the predicted probability distribution \( p(w, x) \) via the forward propagation function \( f(w, x) \) of the neural network, where \( w \) denotes the network's weight parameters. For any sample \( (x,y) \) in the coreset, we input it into both the full-precision and quantized models. The resulting output probability distributions (logits) are represented by:
\begin{equation}
\label{eq:output}
   p_{\mathbf{F}}(w^r,x^r)=f(w^r,x^r),  p_{\mathbf{Q}}(w^q,x^q)=f(w^q,x^q)
\end{equation}

In QAT, a significant challenge is reducing quantization errors, which propagate and accumulate in the deeper layers of the model.  Knowledge distillation (KD), when applied to  QAT, can reduce quantization errors of quantized models ~\cite{mishra2018apprentice,huang2022sdq,mishra2017apprentice,liu2023oscillation}. Typically, the full-precision model acts as the teacher, and the quantized model acts as the student. The objective of the KD loss function is to minimize the disparity between the output distributions of the final layer of the full-precision model and the quantized model. This is achieved through the following formulation:
\begin{equation}
\label{eq:kd}
\mathcal{L}_{KD} = -\sum_{m}^M
p_{\mathbf{T}}^{(m)}(w^r,x_i) \log(p_{\mathbf{Q}}^{(m)}(w^q,x_i))
\end{equation}
where \( \mathcal{L}_{KD} \) is the KD loss function, defined as the cross-entropy between the output distributions \( p_{\mathbf{T}} \) and \( p_{\mathbf{Q}} \) of the full-precision teacher model and the quantized student model, and \( m \) denotes the classes. The teacher model can be a full-precision model corresponding to the quantized model, or it can be a full-precision model of the same type with a deeper number of layers.

However, traditional QAT requires retraining the quantized model using the entire dataset, which leads to substantial computational resources and time overhead. For example, even with a small dataset like CIFAR-100, the QAT of MobileNetV2 takes 75.22 minutes, which is unacceptable for edge devices. Therefore, enhancing the efficiency of QAT is critical.

\begin{figure*}[t]
\centering
\includegraphics[width=1\textwidth]{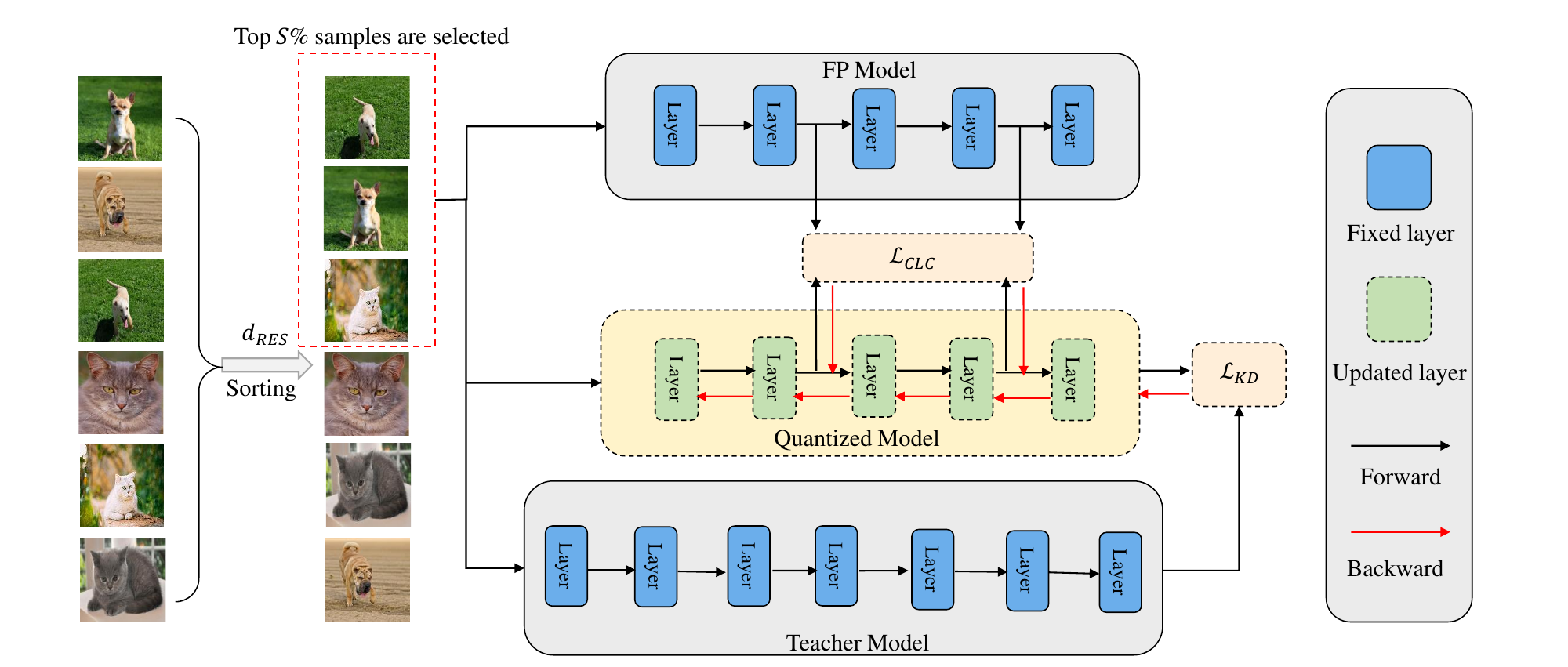} 
\caption{The overview of QuaRC.
}
\label{fig1}
\end{figure*}

\subsection{Coreset Selection for QAT}
Coreset selection techniques can enhance the efficiency of QAT by identifying the most representative subsets for training. However, most existing coreset selection methods are designed for full-precision models and do not account for the unique characteristics of quantized models. For instance, classical methods such as Moderate \cite{xia2023moderate}, Contextual Diversity (CD) \cite{agarwal2020contextual}, and Forgetting \cite{toneva2018an} rely on various criteria to select  samples. Moderate and CD use geometry-based criteria to select samples that can approximate the distribution of the full dataset, while Forgetting focuses on the frequency of forgetting events to identify samples that are more difficult to learn and potentially more informative for training. While these methods perform well for full-precision models, their performance in QAT is even worse than that of random sampling.

Recent work like ACS \cite{huangrobust} attempts to bridge this gap by introducing metrics such as the error vector score and disagreement score.  These metrics aim to identify the samples that contribute the most to the parameter updates of the quantized model. However, they fail to consider the correlation between samples and quantization errors, which can be explicitly manifested as the differences between the output logits of full-precision and quantized models.

\subsection{Challenge and Motivation}
Quantization errors primarily originates from two sources: 1) the rounding and clamping operations during fake quantization, which distort the forward propagation of activations and weights, as reflected in Eq.(\ref{eq:quantization}), and 2) the use of STE during backpropagation, which introduce approximation errors in gradient updates, as reflected in Eq.(\ref{eq:process}). Existing coreset selection criteria fail to prioritize samples that expose these errors during QAT, thereby limiting the model's ability to adapt to quantization noise.

Moreover, quantization errors accumulate and propagate through the intermediate layers of the model. To mitigate these errors, existing QAT methods employ KD strategies during training. However, traditional KD methods focus solely on optimizing the final output. When training with the full dataset, the large volume of data can effectively eliminate the errors in the intermediate layers of the quantized model. But when the coreset size is small (e.g., 10\% of the dataset), even using KD strategies fails to eliminate the errors in the intermediate layers of the quantized model.

The above challenges motivate us to design: 1) a coreset selection method that can \textbf{reflect the quantization errors} of the model, and 2) a training strategy that can \textbf{eliminate the quantization errors} of the model.

\section{Design of QuaRC}
In this section, we first introduce the overview of our method---QuaRC. Subsequently, we describe two key phases of QuaRC: the coreset selection phase and the quantized model training phase.

\subsection{Overview of QuaRC}
As shown in Fig. \ref{fig1}, QuaRC is a framework that utilizes coresets for QAT, which mainly consists of two phases:

\textcircled{1} \textit{Coreset Selection Phase:}  In this phase, we aim to select a small yet representative subset of the dataset that can effectively capture the quantization errors of the model. We propose the \textbf{Relative Entropy Score (RES)} as the criterion for selecting the coreset. Specifically, we input each sample into both the full-precision model and the quantized model, and calculate the relative entropy between their output distributions. Samples with higher relative entropy are considered more capable of capturing the quantization errors. Therefore, we sort all samples based on their relative entropy scores and select the top $S\%$ samples to form the coreset. This coreset will be used for training the quantized model in the next phase.

\textcircled{2} \textit{Quantized Model Training Phase:} In this phase, we train the quantized model using the selected coreset. To address the accumulation of quantization errors in the intermediate layers of the quantized model, we introduce a \textbf{Cascaded Layer Correction (CLC)} strategy. This strategy aligns the output distributions of the intermediate layers of the quantized model with those of the full-precision model by minimizing the Kullback-Leibler (KL) divergence between them. Specifically, we add an additional loss term to the training objective, which encourages the quantized model to learn representations that are closer to those of the full-precision model at each intermediate layer. This helps to reduce the quantization errors and improve the overall accuracy of the quantized model.

We will describe these two phases in detail in Section \ref{res} and Section \ref{clc}.

\begin{figure}[t]
   \centering
   \begin{subfigure}[b]{0.23\textwidth}
       \includegraphics[width=\textwidth]{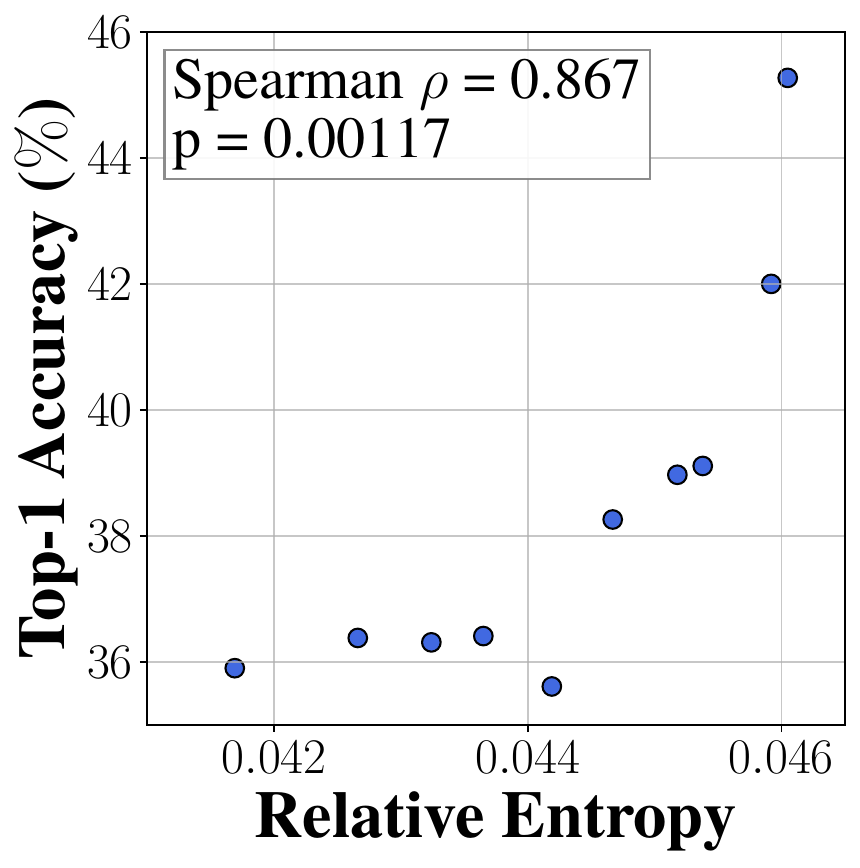}
       \caption{Correlation analysis.}
       \label{fig1:subfig1}
   \end{subfigure}
   \begin{subfigure}[b]{0.23\textwidth}
       \includegraphics[width=\textwidth]{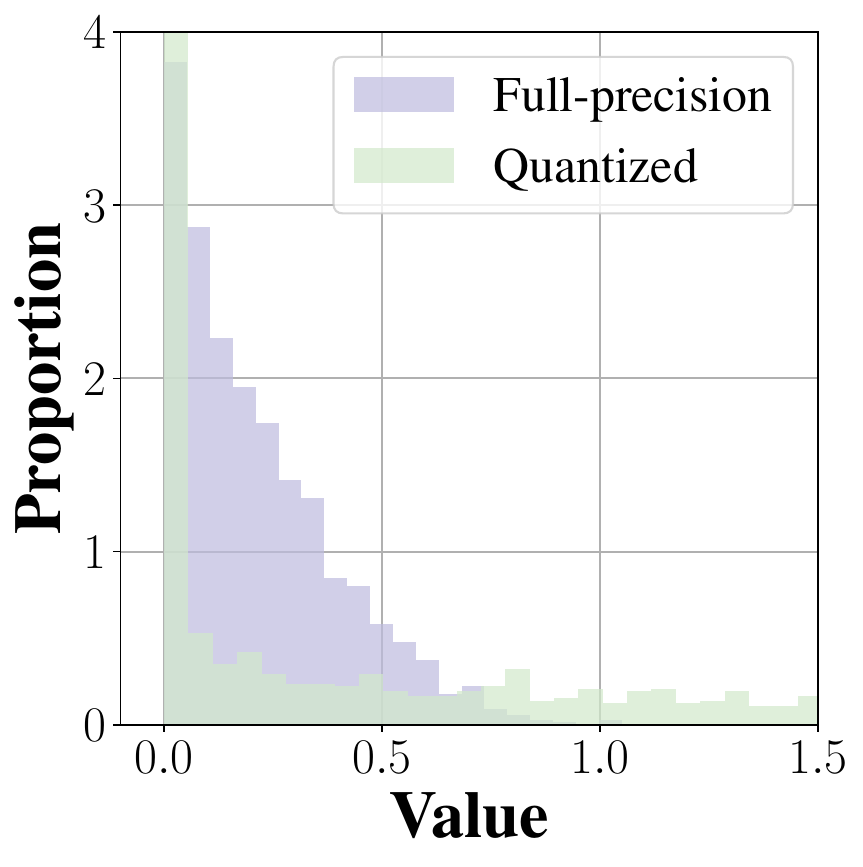}
       \caption{Output distribution analysis.}
       \label{fig1:subfig2}
   \end{subfigure}
    \caption{(a) We select coresets with different mean values of relative entropy for training.
(b) We compare the differences in intermediate layer outputs between the quantized model and the full-precision model. These observations are from training quantized MobileNetV2 on 1\% of CIFAR-100 dataset.}
    \label{fig2}
\end{figure}

\subsection{Relative Entropy Coreset Selection Phase}
\label{res}
\noindent\textbf{Quantization Errors Analysis.} Although existing coreset selection methods have improved the efficiency of QAT, selecting a small-scale coreset for training often leads to a significant drop in the performance of the quantized model. This is because these methods do not directly select the coreset from the perspective of reducing model quantization errors, making it difficult to eliminate such errors in subsequent training. To reduce the model's quantization errors, it is essential to ensure that the selected coreset can reflect those errors. 

Quantization errors manifest as the discrepancy between the outputs of the full-precision model \( p_{\mathbf{F}}(w^r, x^r) \) and its quantized counterpart \( p_{\mathbf{Q}}(w^q, x^q) \) for the same input sample \( (x, y) \). Minimizing this discrepancy is the core objective of QAT (Eq.(\ref{eq:kd})). We can use relative entropy to measure the difference between these two outputs. When the relative entropy is small, it indicates that the output of the quantized model \( p_{\mathbf{Q}}(w^q,x^q) \) is very similar to the output of the full-precision model \( p_{\mathbf{F}}(w^r,x^r) \). This means that the sample hardly captures the errors in the quantized model. In this case, training with such samples results in a small loss value that is close to convergence, leading to very limited updates to the model parameters. Consequently, it is difficult to effectively reduce the quantization errors. Conversely, when the relative entropy is large, it indicates a significant difference between \( p_{\mathbf{Q}}(w^q,x^q) \) and \( p_{\mathbf{F}}(w^r,x^r) \). Training with such samples results in a larger loss value, allowing for substantial updates to the model parameters, thereby effectively reducing the quantization errors.

Furthermore, we select 10 coresets with different relative entropies from the full dataset (the relative entropy of a coreset is defined as the average relative entropy of the samples it contains), with each coreset containing 1\% of the samples from the full dataset. We then train 2-bit quantized MobileNetV2 models using these coresets and obtain the Top-1 accuracy of the models, as shown in Fig.~\ref{fig1:subfig1}. The Spearman correlation coefficient between relative entropy and Top-1 accuracy is 0.867, with \( p = 0.00117 < 0.05 \), indicating a significant and strong positive correlation between these two variables. This experiment further demonstrates that using relative entropy as a metric for coreset selection can effectively reflect the model's quantization errors, helping to reduce these errors during training and thereby improving the model's performance.

\noindent\textbf{Relative Entropy Score.} Based on the above insights, we propose the Relative Entropy as the metric for coreset selection. Before training the quantized model, we first traverse every sample in the dataset, and input each sample into both the quantized model and the full-precision model to obtain two outputs \( p_{\mathbf{Q}}^{(m)}(w^q,x^q) \) and \( p_{\mathbf{F}}^{(m)}(w^r,x^r) \). Then, calculate the Relative Entropy Score (RES) according to the following formula:
\begin{equation}
\label{eq:res}
    d_{\text{RES}} = \sum_{m}^M p_{\mathbf{Q}}^{(m)}(w^q,x^q) \log \left(\frac{p_{\mathbf{Q}}^{(m)}(w^q,x^q)}{p_{\mathbf{F}}^{(m)}(w^r,x^r)}\right)
\end{equation}
where \( m \) denotes the classes.

\noindent\textbf{Final Selection Metric.} The metrics $d_{\text{EVS}} = \| p(w^q_t,x) - y \|_2$  and $d_{\text{DS}} = \| p(w^q_t,x) - p_{\mathbf{T}}(w^r_t,x) \|_2$ select the coreset from the perspective of gradients. Using the cosine annealing weight coefficient \( \alpha(t) = \cos \left(\frac{t}{2T} \pi\right) \), based on the current training epoch \( t \) and the total number of training epochs \( T \), to balance them can select the coreset that contributes the most to model training, thereby maximizing model performance \cite{huangrobust}. Unlike these two metrics, $d_{\text{RES}}$ selects the samples that best reflect the model's quantization errors as the coreset for training, maximizing model performance from the perspective of reducing quantization errors. We select the coreset from both the gradient and quantization error perspectives, combining the above three metrics as the final coreset selection metric:
\begin{equation}
\label{eq:ds}
    d_{\text{S}}(t) = \alpha(t)d_{\text{EVS}}(t) + (1-\alpha(t))d_{\text{DS}}(t) + d_{\text{RES}}(t)
\end{equation}

We sort the samples in the entire dataset based on $d_{\text{S}}(t)$ and select the top $S\%$ of samples for subsequent training.

\subsection{Cascaded Layer Correction Training Phase}
\label{clc}

\noindent\textbf{Intermediate Layer Errors Analysis.} Training with KD on the full dataset can alleviate the accuracy drop of the quantized model. However, when training with a small coreset, it can lead to a significant decrease in the accuracy of the quantized model. To explore the reasons, we evaluate whether the model's learning and representation abilities for samples changed during the quantization process. Specifically, we input the same samples into the full-precision model and the 3-bit quantized model trained using the KD strategy. Throughout the forward propagation, we capture and record outputs from the intermediate layers, with the output distributions depicted in Fig.~\ref{fig1:subfig2}. We find that the inter-layer output distributions of the quantized model and the full-precision model differ significantly. Specifically, the KL divergence between the inter-layer outputs of the 3-bit quantized model and the full-precision model is 9.0e-4, which indicates that the errors in the intermediate layers of the quantized model are difficult to eliminate.

\noindent\textbf{Cascaded Layer Correction.} In common KD training strategies (i.e., optimizing only the output distribution of the last layer), the entire dataset is used for training. Extensive data training can reduce the quantization errors in the model's intermediate layers. However, training with a small-scale coreset is not sufficient to eliminate errors in the intermediate layers. Therefore,  it is necessary to design strategies to optimize the output distribution of the intermediate layers, ensuring that the quantized model maintains a representation ability close to that of the full-precision model at intermediate layers \cite{li2023hard}.

\begin{algorithm}[tb]
\caption{The overall pipeline of QuaRC}
\label{alg:algorithm}
\begin{algorithmic}
\STATE {\bfseries Input:} Training dataset $D =\{(x_i,y_i)\}_{i=1}^n$, Initial coreset $D_{\text{S}}(t)$,  Coreset data fraction $S$, Total training epochs $T$, Selection interval $R$, FP model's weights $\bf W^r$
\STATE {\bfseries Output:} Quantized model's weights $\bf W^q$
\STATE Initialize quantized weights $\bf W^q$ following Eq.(\ref{eq:quantization})
\FOR{$t \in [0, ..., T-1]$}
\STATE \# Stage 1: Coreset selection
\IF{$ t\%R == 0$} 
\FOR{$(x_i,y_i) \in D$}
\STATE Calculate $d_{\text{EVS}}(x_i,t)$ and $d_{\text{DS}}(x_i,t)$
\STATE Calculate $d_{\text{RES}}(x_i,t)$ by Eq.(\ref{eq:res})
\STATE Combine the scores to obtain $d_{\text{S}}(x_i,t)$ by Eq.(\ref{eq:ds})
\ENDFOR
\STATE Sort $d_{\text{S}}(x_i,t)$, choose top $S\%$ samples as $D_{\text{S}}(t)$
\ELSE
\STATE $D_{\text{S}}(t) \gets D_{\text{S}}(t-1)$
\ENDIF
\STATE \# Stage 2: Model training
\STATE Calculate $\mathcal{L}_{\text{KD}}$ by Eq.(\ref{eq:kd}) and $\mathcal{L}_{\text{CLC}}$ by Eq.(\ref{eq:CLC})
\STATE Combine the losses by Eq.(\ref{eq:total}) and train $\bf W^q$ on $D_{\text{S}}(t)$ following Eq.(\ref{eq:total})
\ENDFOR
\end{algorithmic}
\end{algorithm}

To minimize the difference between the intermediate layer outputs of the full-precision and quantized models, we propose a Cascaded Layer Correction (CLC) training strategy.  Specifically, we use a 
full-precision correction model to adjust the output distribution of selected intermediate layers of the quantized model. This adjustment is guided by the following formula:
\begin{equation}
\label{eq:distance}
    \min\sum distance(O, \hat{O})
\end{equation}
where \( O \) represents the output distribution of selected intermediate layers from the full-precision model, and \( \hat{O} \) denotes the output distribution of the corresponding selected intermediate layers from the quantized model. We utilize information distance to measure the distance between output distributions, and reducing the difference in the output of the intermediate layer between the full-precision model and the quantized model is equivalent to optimizing the following loss function in training:
\begin{equation}
\label{eq:CLC}
\mathcal{L}_{CLC}=\sum_{c}^{C}{p_{\mathbf{Q}}^{i}(w^q,x^q)\log{\left(\frac{p_{\mathbf{Q}}^{i}(w^q,x^q)}{p_{\mathbf{F}}^{i}(w^r,x^r)}\right)}}
\end{equation}
where \( c \) represents the intermediate layer to be optimized, \( p_{\mathbf{Q}}^{i}(w^q,x^q) \) and \( p_{\mathbf{F}}^{i}(w^r,x^r) \) denote outputs of the
intermediate layer.

\noindent\textbf{Total Training Loss.} We use both KD and CLC for QAT, and the final loss function is as follows:
\begin{equation}
\label{eq:total}
    \mathcal{L}_{\text{TOTAL}} = \mathcal{L}_{\text{KD}} + \beta\mathcal{L}_{\text{CLC}}
\end{equation}

To help better understand the training process, we provide the brief pseudo-code of it in Algorithm \ref{alg:algorithm}. Following ACS~\cite{huangrobust}, we perform coreset selection every $R$ epoch, where $R$ is predetermined prior to the training.

\section{Theoretical Analysis\label{3.4}}
In this section, we first conduct a theoretical analysis of the computational complexity of QuaRC and demonstrate that it can improve the training efficiency of QAT. Subsequently, we provide a convergence analysis for QAT using coresets.

\subsection{Complexity Analysis}
We theoretically prove that QuaRC can improve the efficiency of QAT. In traditional QAT, the model performs one forward pass and one backward pass on the entire dataset in each training epoch. Assuming the dataset contains \( N \) samples, the number of training epochs is \( T \), the complexity of the forward pass is \( O(F) \), and the complexity of the backward pass is \( O(B) \), the total time complexity of traditional QAT is \(O(T \cdot N \cdot (F + B))\).

QuaRC selects a small subset of samples for training using a coreset. Assuming the selected sample ratio is \( S \) (e.g., \( S = 0.01 \) represents selecting 1\% of the samples), the complexity of the forward and backward passes per epoch is reduced to \( O(S \cdot N \cdot (F + B)) \). Therefore, the complexity of the coreset-based training is \( O(T \cdot S \cdot N \cdot (F + B))\). RES requires two additional forward passes to compute the relative entropy. For each sample in the dataset, these two forward passes are performed. 
The computational complexity of calculating the relative entropy is negligible compared to forward propagation. Assume that coreset selection is performed every 
\( R \) rounds. Thus, the computational complexity of RES is \( O(2 \cdot F \cdot N \cdot T / R )
\). The introduction of CLC mainly involves adding an additional loss term to optimize the output distribution of intermediate layers. The computational complexity of CLC is negligible and already included in the complexity of the forward pass. By summing up the above complexities, the total time complexity of QuaRC is \( O(T \cdot S \cdot N \cdot (F + B)) + O(2 \cdot F \cdot N  \cdot T / R)\). Since \( O(B) >\)  \( O(F) \) \cite{PyTorch}, when \( S \)  is sufficiently small  (e.g., \( S <\) 0.1), the following equation holds:

\begin{equation}
\label{scs}
O(T \cdot N \cdot (F + B)) > > O(T \cdot S \cdot N \cdot (F + B)) + O(2 \cdot F \cdot N  \cdot T / R)
\end{equation}

Through the above theoretical analysis, it can be proven that using our method can improve the efficiency of QAT.

\subsection{Convergence Analysis}
We analyze the convergence of QAT with coresets. For the coreset $D_{\text{S}}$ and the loss function $\mathcal{L}(w^q; D_{\text{S}})$, we assume that:

\begin{enumerate}
    \item \textit{Assumption 1:\label{assumption1}} The loss function $\mathcal{L}(w^q; D_{\text{S}})$ is Lipschitz-smooth which implies that there exists a constant $L > 0$ such that for all $w^q, w^{q'} \in \mathbb{R}^d$ and any sample $d \in D_{\text{S}}$,
    \begin{equation}
    \label{eq:proof1}
    \begin{aligned}
    \| \nabla \mathcal{L}(w^q; d) - \nabla \mathcal{L}(w^{q'}; d) \| & \leq L \| w^q - w^{q'} \|
    \end{aligned}
    \end{equation}
    \item \textit{Assumption 2:\label{assumption2}} The gradient $\nabla \mathcal{L}(w^q; d)$ is bounded, i.e., there exists a constant $G$ such that for all $w^q \in \mathbb{R}^d$ and $d \in D_{\text{S}}$ we have
    \begin{equation}
    \label{eq:proof2}
    \begin{aligned}
    \| \nabla \mathcal{L}(w^q; d) \| & \leq G
    \end{aligned}
    \end{equation}
    \item  \textit{Assumption 3:\label{assumption3}} The learning rate $\eta_t$ satisfies
\begin{equation}
\label{eq:proof3}
    \sum_{t=1}^{\infty} \eta_t = \infty \quad \text{and} \quad \sum_{t=1}^{\infty} \eta_t^2 < \infty
\end{equation}
\end{enumerate}

Stochastic gradient descent (SGD) is employed for optimization. To simplify the proof, the weights can be updated according to the following equation:
\begin{equation}
\label{eq:proof4}
w^q_{t+1} = w^q_t - \eta_t g_t
\end{equation}
where $g_t = \nabla \mathcal{L}(w^q_t; d)$ and $d \in D_{\text{S}}$.

\begin{lemma}\label{lemma1}
For any iteration $t$ and $d \in D_{\text{S}}$, the loss difference satisfies:
\begin{equation}
\label{eq:loss_change}
\mathcal{L}(w^q_{t+1}; d) - \mathcal{L}(w^q_t; d) \leq -\eta_t \left(1 - \frac{L\eta_t}{2}\right) \|g_t\|^2
\end{equation}
\end{lemma}

\begin{proof}
Expand $\mathcal{L}(w^q_{t+1}; d)$ using Taylor series around $w^q_t$:
\begin{equation}
\begin{aligned}
\mathcal{L}(w^q_{t+1}; d) &= \mathcal{L}(w^q_t; d) + \nabla \mathcal{L}(w^q_t; d)^T (w^q_{t+1} - w^q_t) \\
&\quad + \frac{1}{2}(w^q_{t+1} - w^q_t)^T \nabla^2 \mathcal{L}(c)(w^q_{t+1} - w^q_t)
\end{aligned}
\end{equation}
where \( c \) lies between \( w^q_t \) and \( w^q_{t+1} \). Substitute $w^q_{t+1} - w^q_t = -\eta_t g_t$ and use \textit{Assumption~\ref{assumption1}}:
\begin{equation}
\begin{aligned}
\mathcal{L}(w^q_{t+1}; d) &\leq \mathcal{L}(w^q_t; d) - \eta_t \|g_t\|^2 + \frac{1}{2}\eta_t^2 L \|g_t\|^2 \\
&= \mathcal{L}(w^q_t; d) - \eta_t \left(1 - \frac{L\eta_t}{2}\right) \|g_t\|^2
\end{aligned}
\end{equation}
\end{proof}

\begin{lemma}\label{lemma2}
The sequence $\{\mathbb{E}[\mathcal{L}(w^q_t; d)]\}$ is non-increasing:
\begin{equation}
\label{eq:monotonicity}
\mathbb{E}[\mathcal{L}(w^q_{t+1}; d)] \leq \mathbb{E}[\mathcal{L}(w^q_t; d)]
\end{equation}
\end{lemma}

\begin{proof}
Take expectations in Lemma~\ref{lemma1} and use \textit{Assumption~\ref{assumption2}}:
\begin{equation}
\mathbb{E}[\mathcal{L}(w^q_{t+1}; d) - \mathcal{L}(w^q_t; d)] \leq -\eta_t \left(1 - \frac{L\eta_t}{2}\right) \mathbb{E}[\|g_t\|^2]
\end{equation}
From \textit{Assumption~\ref{assumption3}}, $\eta_t \in [0, 2/L]$ for large $t$, thus:
\begin{equation}
\mathbb{E}[\mathcal{L}(w^q_{t+1}; d)] \leq \mathbb{E}[\mathcal{L}(w^q_t; d)]
\end{equation}
\end{proof}

\begin{theorem}\label{thm1}
Under Assumptions~\ref{assumption1}-\ref{assumption3}, the sequence $\{\mathbb{E}[\mathcal{L}(w^q_t; d)]\}$ converges to a limit $L^* \geq 0$.
\end{theorem}

\begin{proof}
From Lemma~\ref{lemma2}, the sequence $\{\mathbb{E}[\mathcal{L}(w^q_t; d)]\}$ is:
\begin{itemize}
\item Non-increasing: $\mathbb{E}[\mathcal{L}(w^q_{t+1}; d)] \leq \mathbb{E}[\mathcal{L}(w^q_t; d)]$
\item Lower bounded by 0: $\mathcal{L}(w^q_t; d) \geq 0$ (by definition of loss functions)
\end{itemize}
By the Monotone Convergence Theorem, there exists $L^* \geq 0$ such that:
\begin{equation}
\lim_{t \to \infty} \mathbb{E}[\mathcal{L}(w^q_t; d)] = L^*
\end{equation}
\end{proof}

Therefore,  the sequence $\{\mathbb{E}[\mathcal{L}(w^q_t; d)]\}_{t=1}^\infty$ is non-increasing and lower-bounded, so it must converge to some limit $L^*$,  signifying the convergence of QAT with coresets.

\section{Experiments}
\subsection{Experimental Setup}

\noindent\textbf{Datasets and Networks.} The datasets used in our experiments are CIFAR-100 \cite{krizhevsky2009learning} and ImageNet-1K \cite{deng2009imagenet}. We apply data augmentation techniques, \textit{RandomResizedCrop} and \textit{RandomHorizontalFlip}, provided by PyTorch \cite{paszke2019pytorch}. Consistent with ACS \cite{huangrobust}, we evaluate our method on two widely-used neural networks: MobileNetV2 \cite{howard2017mobilenets} and ResNet-18 \cite{he2016deep}.

\noindent\textbf{Baselines.} Our work focuses on addressing performance degradation in quantized models using small-scale coresets for QAT. Therefore, we selected baselines related to coreset selection, rather than QAT methods that modify the training strategy without addressing coreset selection. To ensure a fair comparison, we have selected multiple baselines for comparison, which can be categorized as follows:
\begin{itemize}
    \item Random Sampling: Randomly select a coreset of specified size from each class for QAT.
    \item Forgetting\cite{toneva2018an}: Select samples based on early training forgetting statistics, focusing on those triggering more forgetting events and being more challenging to learn, as the coreset for QAT.
    \item ContextualDiversity(CD)\cite{agarwal2020contextual}: Choose the centroids of sample clusters as the coreset for QAT.
    \item Moderate\cite{xia2023moderate}: Select samples with scores close to the median score as the coreset for QAT.
    \item ACS\cite{huangrobust}: Adaptively select samples with the greatest impact on gradients as the coreset for QAT at different intervals.
\end{itemize}

\begin{table*}[t]
\centering
\caption{Comparison of Top-1 
  and Top-5 accuracy of different methods on QAT of quantized
MobileNetV2 on CIFAR-100 with different subset fractions. The bitwidth for quantized MobileNetV2 is
2/32, 3/32, and 4/32 for weights/activations.}
\label{tab1}
\begin{tabular}{c|c|cc|cc|cc|cc|cc}
\toprule    
\multirow{2}[4]{*}{Method} & \multirow{2}[4]{*}{Bit Width} & \multicolumn{2}{c|}{$1\%$} & \multicolumn{2}{c|}{$3\%$} & \multicolumn{2}{c|}{$5\%$} & \multicolumn{2}{c|}{$7\%$} & \multicolumn{2}{c}{$10\%$} \\
\cmidrule{3-12}
 &   & Top-1 & Top-5 & Top-1 & Top-5 & Top-1 & Top-5 & Top-1 & Top-5 & Top-1 & Top-5\\
    \midrule
    Random & \multirow{6}[1]{*}{2w32a} &40.57  &71.74  &53.22  &80.43  &56.21  &83.45  &59.52  &85.73  &60.81  &86.42\\
    Forgetting\cite{toneva2018an} & &40.73  &70.69  &51.90  &80.04  &55.78  &83.27  &58.26  &84.36  &60.39  &86.01\\
    CD\cite{agarwal2020contextual}  & &38.61  &68.30  &51.53  &79.36  &55.13  &82.55     &57.86  &83.63  &58.53  &85.05\\
    Moderate\cite{xia2023moderate}  & &35.14  &65.47  &48.99  &77.86  &51.99  &80.91  &55.17  &83.31  &57.51  &84.02 \\
    ACS\cite{huangrobust}  & &46.84  &78.09  &57.07  &84.64  &59.59  &85.95  &61.59  &87.54  &63.32  &87.71\\
    QuaRC (Ours)  & &\cellcolor{gray!20}\textbf{56.36}  &\cellcolor{gray!20}\textbf{85.08}  &\cellcolor{gray!20}\textbf{61.97}  &\cellcolor{gray!20}\textbf{87.66}  &\cellcolor{gray!20}\textbf{63.30}  &\cellcolor{gray!20}\textbf{88.01}   &\cellcolor{gray!20}\textbf{64.19}  &\cellcolor{gray!20}\textbf{88.62}  &\cellcolor{gray!20}\textbf{65.64}  &\cellcolor{gray!20}\textbf{88.98}\\
    \midrule
    Random & \multirow{6}[1]{*}{3w32a} &62.93  &86.89  &65.59  &88.86  &67.58  &89.72  &68.15  &90.34  &69.08  &90.34\\
    Forgetting\cite{toneva2018an}  & &63.01  &86.45  &65.43  &89.44  &67.58   &89.41  &68.18  &89.85  &68.77  &90.86\\
    CD\cite{agarwal2020contextual}  & &62.10  &86.31  &66.16  &88.50  &67.24  &89.24   &68.52  &89.81  &68.77  &90.41\\
    Moderate\cite{xia2023moderate}  & &61.71  &86.19  &65.29  &88.26  &67.00  &89.41  &67.57  &89.97 &67.88  &90.07 \\
    ACS\cite{huangrobust}  & &64.98  &88.64  &67.83  &90.11  &68.48  &90.58   &68.94  &91.22  &69.13  &90.59\\
    QuaRC (Ours)  & &\cellcolor{gray!20}\textbf{68.69}  &\cellcolor{gray!20}\textbf{90.89}  &\cellcolor{gray!20}\textbf{69.41}  &\cellcolor{gray!20}\textbf{91.27}  &\cellcolor{gray!20}\textbf{70.11}  &\cellcolor{gray!20}\textbf{91.77}   &\cellcolor{gray!20}\textbf{70.23}  &\cellcolor{gray!20}\textbf{91.35}  &\cellcolor{gray!20}\textbf{70.99}  &\cellcolor{gray!20}\textbf{91.60}\\
    \midrule
    Random & \multirow{6}[1]{*}{4w32a} &69.21  &90.55  &70.17  &91.28  &70.68  &91.10  &70.94  &91.31  &71.28  &91.80\\
    Forgetting\cite{toneva2018an}  & &69.35  &90.38  &70.19  &91.12  &70.47  &91.35  &71.12  &91.51  &71.18  &91.84\\
    CD\cite{agarwal2020contextual}  & &69.03  &90.41  &70.35  &91.04  &70.53  &91.16   &71.05  &91.58  &71.30  &91.29\\
    Moderate\cite{xia2023moderate}  & &68.69  &90.11  &70.30  &90.74  &70.40  &91.15  &70.86  &91.06  &71.02  &91.34\\
    ACS\cite{huangrobust}  & &69.37  &90.68  &70.55  &91.30  &70.92  &91.50   &71.26  &91.53  &71.39  &91.83\\
    QuaRC (Ours)  & &\cellcolor{gray!20}\textbf{71.25}  &\cellcolor{gray!20}\textbf{91.39}  &\cellcolor{gray!20}\textbf{71.53}  &\cellcolor{gray!20}\textbf{91.64}  &\cellcolor{gray!20}\textbf{71.86}  &\cellcolor{gray!20}\textbf{91.73}  &\cellcolor{gray!20}\textbf{71.98}  &\cellcolor{gray!20}\textbf{91.72}  &\cellcolor{gray!20}\textbf{72.03}  &\cellcolor{gray!20}\textbf{91.94}\\
    \bottomrule
    \end{tabular}
\end{table*}

\noindent\textbf{Metrics.} We use Top-1 and Top-5 accuracy as evaluation metrics. For inference, we report the Top-1 and Top-5 accuracy of the full-precision models. Specifically, the full-precision MobileNetV2 achieves 72.56\% Top-1 and 91.93\% Top-5 accuracy on CIFAR-100, while the full-precision ResNet-18 achieves 69.76\% Top-1 and 89.07\% Top-5 accuracy on ImageNet-1K.

\noindent\textbf{Implementation Details.} Following ACS\cite{huangrobust},  we use the quantization method from LSQ+ \cite{bhalgat2020lsq+}. For MobileNetV2 on CIFAR-100, we train for 200 epochs with a learning rate of 0.01, weight decay of 5e-4, batch size of 256, and the SGD optimizer, setting $R=50$ and $\beta= 1e5$. The teacher model is MobileNetV2. For ResNet-18 on ImageNet-1K, we train for 120 epochs with a learning rate of 1.25e-3, no weight decay, batch size of 128, and the Adam optimizer, setting $R=10$ and $\beta= 3e3$. The teacher model is ResNet-101. All experiments are conducted on an NVIDIA RTX 4090 GPU.

\subsection{Main Results}
We present the  Top-1 and Top-5 accuracy for MobileNetV2 on CIFAR-100 and ResNet-18 on ImageNet-1K in Table \ref{tab1} and Table \ref{tab2}. For MobileNetV2 on CIFAR-100, we quantize the weights to 2/3/4-bit, with activations remaining in full-precision. Specifically, when the weights are quantized to 2-bit, random sampling achieves a Top-1 accuracy of only \(40.57\%\) on a \(1\%\) subset. Other methods such as Forgetting, CD, and Moderate perform even worse, with accuracies of \(40.73\%\), \(38.61\%\), and \(35.14\%\), respectively, all falling below the baseline level of random selection. This indicates that these methods fail to account for the characteristics of quantization and are not suitable for QAT. ACS, which employs a gradient-based dynamic balancing strategy, improves performance to \(46.84\%\), but there is still a significant gap compared to the full-precision model's accuracy of \(72.56\%\). In contrast, our proposed QuaRC achieves a Top-1 accuracy of \(56.36\%\) on the \(1\%\) subset, representing an absolute improvement of \(9.52\%\) over ACS. This fully demonstrates its effectiveness in reducing quantization errors. As the quantization bitwidth increases to 3-bit and 4-bit, QuaRC continues to maintain a significant advantage. Specifically, on the \(1\%\) subset, QuaRC achieves Top-1 accuracies of \(68.69\%\) and \(71.25\%\), respectively, which are significantly higher than ACS's \(64.98\%\) and \(69.37\%\). Moreover, QuaRC also shows excellent performance on higher subset proportions (such as \(5\%\) and \(10\%\)), with Top-1 accuracies reaching \(63.30\%\) and \(65.64\%\) under 2-bit setting, respectively. This further proves the robustness and effectiveness of the QuaRC under different data scales.

For ResNet-18 on ImageNet-1K, we quantize both weights and activations to 2/3/4-bit. On large-scale datasets, QuaRC also demonstrates significant advantages. Specifically, under the setting of 2-bit weight and activation quantization, random sampling achieves a Top-1 accuracy of only \(29.12\%\) on a \(1\%\) subset. Other methods such as Forgetting and CD, due to their insufficient sample representativeness, further degrade performance, with the Moderate method's accuracy dropping to \(22.29\%\). The ACS method improves performance to \(40.62\%\), but there is still a significant issue of error accumulation. QuaRC achieves a Top-1 accuracy of \(46.34\%\) on the \(1\%\) subset, representing an absolute improvement of \(5.72\%\) over ACS. This validates the synergistic effect of the quantization-aware coreset selection and cascaded layer correction strategies. As the subset proportion increases to \(10\%\), the method maintains a stable performance of \(62.32\%\) under 2-bit quantization. Moreover, under quantization settings of 3-bit and 4-bit, QuaRC achieves absolute improvements of \(2.82\%\) and \(1.47\%\) over ACS on the \(1\%\) subset, respectively. These results comprehensively demonstrate the advantages of QuaRC.

As the bitwidth and subset fraction increase, the performance enhancement of the quantized model using our method diminishes. This occurs because larger bitwidth and subset fraction bring the performance of the quantized model closer to that of the full-precision model, making additional improvements more challenging. In general, our method can improve the performance of quantized models under various bitwidths and subset fractions, and the experiments have validated the superiority of our method.

\begin{table*}[t]
  \centering
 \caption{Comparison of Top-1 
  and Top-5 accuracy of different methods on QAT of quantized
ResNet-18 on ImageNet-1K with different subset fractions. The bitwidth for quantized ResNet-18 is
2/2, 3/3, and 4/4 for weights/activations.}
\begin{tabular}{c|c|cc|cc|cc|cc|cc}
\toprule    
\multirow{2}[4]{*}{Method} & \multirow{2}[4]{*}{Bit Width} & \multicolumn{2}{c|}{$1\%$} & \multicolumn{2}{c|}{$3\%$} & \multicolumn{2}{c|}{$5\%$} & \multicolumn{2}{c|}{$7\%$} & \multicolumn{2}{c}{$10\%$} \\
\cmidrule{3-12}
 &   & Top-1 & Top-5 & Top-1 & Top-5 & Top-1 & Top-5 & Top-1 & Top-5 & Top-1 & Top-5\\
    \midrule
    Random & \multirow{6}[1]{*}{2w2a} &29.12  &53.95  &45.60  &71.09  &51.28  &75.89  &54.59  &78.58  &57.44  &80.76\\
    Forgetting\cite{toneva2018an}  & &30.13  &54.86  &45.45  &70.10  &50.42  &74.29  &52.92  &76.49  &56.15  &79.20\\
    CD\cite{agarwal2020contextual}  & &25.85  &50.20  &43.62  &69.45  &49.99  &75.52     &54.09  &78.62  &57.33  &81.13\\
    Moderate\cite{xia2023moderate}  & &22.29  &45.77  &40.68  &67.43  &47.39  &73.55  &51.74  &77.17  &55.56  &79.98 \\
    ACS\cite{huangrobust}  & &40.62  &65.60  &53.68  &77.22  &57.39  &80.51  &59.76  &82.41  &61.55  &83.44\\
    QuaRC (Ours)  & &\cellcolor{gray!20}\textbf{46.34}  &\cellcolor{gray!20}\textbf{71.08}  &\cellcolor{gray!20}\textbf{56.22}  &\cellcolor{gray!20}\textbf{79.46}  &\cellcolor{gray!20}\textbf{59.81}  &\cellcolor{gray!20}\textbf{82.50}   &\cellcolor{gray!20}\textbf{61.31}  &\cellcolor{gray!20}\textbf{83.50}  &\cellcolor{gray!20}\textbf{62.32}  &\cellcolor{gray!20}\textbf{84.11}\\
    \midrule
    Random & \multirow{6}[1]{*}{3w3a} &48.39  &73.88  &57.02  &80.64  &59.58  &82.35  &60.97  &83.65  &62.49  &84.60\\
    Forgetting\cite{toneva2018an}  & &48.13  &73.09  &55.77  &79.26  &58.60   &81.14  &60.16  &82.41  &61.42  &83.23\\
    CD\cite{agarwal2020contextual}  & &46.78  &72.75  &56.20  &80.43  &59.09  &82.68   &60.74  &83.76  &62.39  &84.72\\
    Moderate\cite{xia2023moderate}  & &45.10  &72.02  &54.68  &79.57  &57.77  &81.72  &59.46  &82.86  &61.05  &84.02 \\
    ACS\cite{huangrobust}  & &58.39  &81.30  &63.18  &84.79  &64.51  &85.82   &65.47  &86.63  &65.99  &\cellcolor{gray!20}\textbf{86.98}\\
    QuaRC (Ours)  & &\cellcolor{gray!20}\textbf{61.21}  &\cellcolor{gray!20}\textbf{83.21}  &\cellcolor{gray!20}\textbf{64.49}  &\cellcolor{gray!20}\textbf{85.59}  &\cellcolor{gray!20}\textbf{65.40}  &\cellcolor{gray!20}\textbf{86.22}   &\cellcolor{gray!20}\textbf{65.67}  &\cellcolor{gray!20}\textbf{86.66}  &\cellcolor{gray!20}\textbf{66.19}  &86.81\\
    \midrule
    Random & \multirow{6}[1]{*}{4w4a} &56.59  &80.45  &60.89  &83.53  &62.80  &84.73  &63.78  &85.38  &64.47  &86.03\\
    Forgetting\cite{toneva2018an}  & &55.97  &79.75  &59.98  &82.51  &62.06  &83.72  &63.08  &84.41  &64.03  &85.15\\
    CD\cite{agarwal2020contextual}  & &56.26  &80.32  &60.40  &83.38  &62.40  &84.88   &63.68  &85.81  &64.72  &86.25\\
    Moderate\cite{xia2023moderate}  & &54.87  &79.85  &59.43  &83.28  &61.33  &84.29  &62.58  &85.01  &63.60  &85.80\\
    ACS\cite{huangrobust}  & &63.83  &85.19  &66.41  &87.04  &67.29  &87.74   &67.76  &\cellcolor{gray!20}\textbf{87.92}  &68.23  &\cellcolor{gray!20}\textbf{88.34}\\
    QuaRC (Ours)  & &\cellcolor{gray!20}\textbf{65.70}  &\cellcolor{gray!20}\textbf{86.31}  &\cellcolor{gray!20}\textbf{67.21 } &\cellcolor{gray!20}\textbf{87.54 } &\cellcolor{gray!20}\textbf{67.42}  &\cellcolor{gray!20}\textbf{87.77}  &\cellcolor{gray!20}\textbf{67.81}  &87.77  &\cellcolor{gray!20}\textbf{68.30}  &88.09\\
    \bottomrule
    \end{tabular}
    \label{tab2}
\end{table*}

\begin{table}[t]
\centering
\caption{Evaluating RES and CLC of our method. We report the Top-1 and Top-5 accuracy of  quantized MobileNetV2 on CIFAR-100. The baseline is ACS.}
\begin{tabular}{c|cc|cc}
\toprule 
Method         & Bit width  &Fraction   & Top-1  & Top-5 \\ \toprule 
Baseline           & 2w32a    &$1\%$    &46.84      &78.09      \\
+RES           & 2w32a    &$1\%$    &48.77      &78.94       \\
+CLC           & 2w32a    &$1\%$    &55.25      &84.03       \\
\hspace{1em}+RES+CLC\hspace{1em}           &2w32a      &$1\%$    &\cellcolor{gray!20}\textbf{56.36}     &\cellcolor{gray!20}\textbf{85.08}      \\
\midrule
Baseline           & 3w32a    &$1\%$    &64.98      &88.64      \\
+RES           & 3w32a    &$1\%$    &65.90      &89.29       \\
+CLC           & 3w32a    &$1\%$    &68.32      &90.53       \\
+RES+CLC           &3w32a      &$1\%$    &\cellcolor{gray!20}\textbf{68.69}     &\cellcolor{gray!20}\textbf{90.89}      \\
\midrule
Baseline           & 4w32a    &$1\%$    &69.37      &90.68      \\
+RES           & 4w32a    &$1\%$    &69.57      &90.77       \\
+CLC           & 4w32a    &$1\%$    &70.94      &91.22       \\
+RES+CLC           &4w32a      &$1\%$    &\cellcolor{gray!20}\textbf{71.25}     &\cellcolor{gray!20}\textbf{91.39}      \\
\midrule
\end{tabular}
\label{tab3}
\end{table}

\subsection{Ablation Studies}
\subsubsection{Effect of Different Components} Table \ref{tab3} demonstrates the impact of different components (RES and CLC) on the performance of the quantized MobileNetV2 model on the CIFAR-100 dataset. It is evident from the table that using either RES or CLC alone can significantly improve the Top-1 and Top-5 accuracy of the model. However, the most substantial performance gains are achieved when both components are combined. Specifically, under the 2-bit weight quantization and 32-bit full-precision activation setting, using RES alone can increase the Top-1 accuracy from 46.84\% to 48.77\% and the Top-5 accuracy from 78.09\% to 78.94\%. Using CLC alone can boost the Top-1 accuracy to 55.25\% and the Top-5 accuracy to 84.03\%. When both RES and CLC are used together, the Top-1 accuracy further increases to 56.36\% and the Top-5 accuracy to 85.08\%. This indicates that RES and CLC are complementary in reducing quantization errors, and their combined use can more effectively enhance the performance of quantized models.

The same trend persists under higher bitwidth quantization settings. For instance, under the 3-bit weight quantization and 32-bit full-precision activation setting, using RES alone can increase the Top-1 accuracy from 64.98\% to 65.90\% and the Top-5 accuracy from 88.64\% to 89.29\%. Using CLC alone can boost the Top-1 accuracy to 68.32\% and the Top-5 accuracy to 90.53\%. When both RES and CLC are combined, the Top-1 accuracy further increases to 68.69\% and the Top-5 accuracy to 90.89\%. This further proves the effectiveness of RES and CLC across different quantization bitwidths.

\begin{table}[t]
\centering
\caption{Coreset selection metric analysis. We report the
Top-1 and Top-5 accuracy of quantized MobileNetV2 on CIFAR-
100.}
\begin{tabular}{c|ccc|cc}
\toprule 
Model         &$d_{\text{EVS}}$   &$d_{\text{DS}}$  &$d_{\text{RES}}$  & Top-1  & Top-5 \\ \toprule 
\multirow{7}{*}{MobileNetV2} %
  & \checkmark &     &       &53.87       &82.44     \\
  &            & \checkmark &       &54.05       &82.79     \\
  &            &     & \checkmark &55.41       &83.23     \\
  & \checkmark & \checkmark &       &55.25       &84.03     \\
  & \checkmark &     & \checkmark &56.18       &84.25     \\
  &            & \checkmark & \checkmark &55.33       &83.71     \\
  & \checkmark & \checkmark & \checkmark &\cellcolor{gray!20}\textbf{56.36}       &\cellcolor{gray!20}\textbf{85.08 }    \\
\midrule
\end{tabular}
\label{tab4}
\end{table}

\subsubsection{Effect of Coreset Selection Metric} Table \ref{tab4} provides a detailed analysis of the impact of different coreset selection metrics on the performance of the quantized MobileNetV2 model on the CIFAR-100 dataset. The results demonstrate that the choice of selection metric plays a crucial role in determining the effectiveness of the coreset in capturing the essential features of the dataset and, consequently, the performance of the quantized model. Specifically, the table shows the performance of the model when using individual metrics (\(d_{\text{EVS}}\), \(d_{\text{DS}}\), and \(d_{\text{RES}}\)) as well as their combinations. 

When using a single metric for coreset selection, \(d_{\text{RES}}\) (Relative Entropy Score) achieves the highest Top-1 accuracy of 55.42\%, outperforming the other individual metrics. This indicates that \(d_{\text{RES}}\) is particularly effective in capturing the quantization errors, which is a critical factor for improving the performance of the quantized model. The metric \(d_{\text{RES}}\) directly measures the difference between the output distributions of the quantized and full-precision models, making it a strong indicator of how well the coreset can represent the quantization errors. In contrast, \(d_{\text{EVS}}\) and \(d_{\text{DS}}\) focus on gradient-based selection, which is also important but less effective on its own in capturing the quantization errors. For example, using \(d_{\text{EVS}}\) alone results in a Top-1 accuracy of 53.87\%, while \(d_{\text{DS}}\) alone achieves 54.05\%. Although these metrics are useful for identifying samples that contribute significantly to the gradient updates, they do not directly address the quantization errors, which are the primary source of performance degradation in quantized models. The most significant performance improvement is observed when all three metrics (\(d_{\text{EVS}}\), \(d_{\text{DS}}\), and \(d_{\text{RES}}\)) are combined. The combined metric achieves a Top-1 accuracy of 56.36\%, which is higher than any individual metric. This suggests that a comprehensive coreset selection strategy that considers both gradient information and quantization errors is more effective in improving the performance of the quantized model. The combined metric leverages the strengths of each individual metric: \(d_{\text{EVS}}\) and \(d_{\text{DS}}\) provide information about the gradient updates, while \(d_{\text{RES}}\) ensures that the coreset captures the quantization errors effectively. By integrating these aspects, the coreset becomes more representative of the full dataset, leading to better training efficiency and higher accuracy of the quantized model.

\begin{figure*}
\centering
\includegraphics[width=1\textwidth]{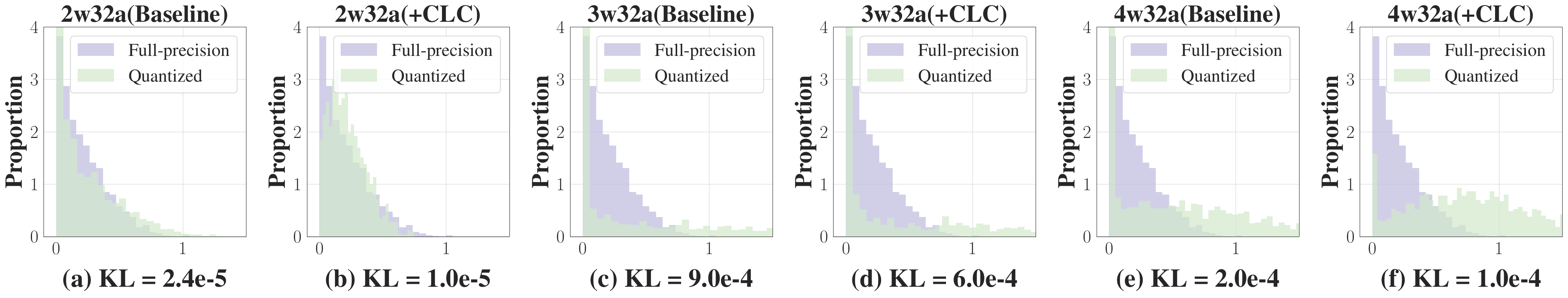} 
\caption{Comparison of the selected inter-layer output distributions between the full-precision MobileNetV2 model and the quantized model with different quantization settings on CIFAR-100. }
\label{fig3}
\end{figure*}

\subsubsection{Visualization of Inter-layer Output} Fig. \ref{fig3} illustrates the distribution of intermediate layer outputs for the quantized and full-precision models, highlighting the effectiveness of the proposed Cascaded Layer Correction (CLC) strategy in reducing quantization errors across different quantization bitwidth settings. The figure compares the output distributions of the selected intermediate layer (the second-to-last layer) of MobileNetV2 on the CIFAR-100 dataset, with weights quantized to 2/3/4-bit precision, while activations remain in full precision. The KL divergence values shown in each subplot indicate the difference between the output distributions of the quantized and full-precision models.

In the case of 2-bit weight quantization, the baseline quantized model exhibits a KL divergence of 2.4e-5 compared to the full-precision model. After applying the CLC strategy, the KL divergence is reduced to 1.0e-5, indicating significant alignment of the output distributions. This demonstrates that the CLC strategy effectively mitigates the quantization errors introduced by low-bit weight precision. For 3-bit weight quantization, the baseline KL divergence is 9.0e-4. However, the CLC strategy reduces this divergence to 6.0e-4. Similarly, for 4-bit weight quantization, the baseline KL divergence is 2.0e-4, and the CLC strategy further reduces it to 1.0e-4. The consistent reduction in KL divergence across different quantization bitwidths confirms the effectiveness of the CLC strategy in aligning intermediate layer outputs, thereby improving the overall accuracy and performance of the quantized model.

\begin{table}[t]
\centering
\caption{The results of combining CLC with different coreset selection methods. we report
the Top-1 and Top-5 accuracy of quantized MobileNetV2 on
CIFAR-100. }
\begin{tabular}{c|cc|cc}
\toprule 
Method         & Bit width  &Fraction   & Top-1  & Top-5 \\ \toprule 
Random           & 2w32a    &$1\%$    &40.57      &71.74      \\
Random+CLC           &2w32a      &$1\%$    &\cellcolor{gray!20}\textbf{49.32}     &\cellcolor{gray!20}\textbf{79.01}      \\
\midrule
Forgetting           & 2w32a    &$1\%$     &40.73       &70.69      \\
Forgetting+CLC           &2w32a      &$1\%$     &\cellcolor{gray!20}\textbf{47.29}      &\cellcolor{gray!20}\textbf{77.96}      \\
\midrule
CD           & 2w32a     &$1\%$     &38.61      &68.30      \\
CD+CLC           &2w32a      &$1\%$     &\cellcolor{gray!20}\textbf{48.53}      &\cellcolor{gray!20}\textbf{78.59}      \\
\midrule
Moderate           & 2w32a     &$1\%$     &35.14      &65.47      \\
Moderate+CLC           &2w32a      &$1\%$     &\cellcolor{gray!20}\textbf{46.40}      &\cellcolor{gray!20}\textbf{78.21}      \\
\midrule
ACS           &2w32a     &$1\%$    &46.84       &78.09      \\
ACS+CLC           &2w32a      &$1\%$     &\cellcolor{gray!20}\textbf{55.25}      &\cellcolor{gray!20}\textbf{84.03}      \\
\midrule
\end{tabular}
\label{tab5}
\end{table}

\subsubsection{CLC with Different Coreset Selection Methods} 
Table \ref{tab5} provides a comprehensive evaluation of the CLC strategy when integrated with various coreset selection methods. The results highlight the versatility and effectiveness of the CLC approach in enhancing the performance of quantized models across different coreset selection techniques. Specifically, the combination of CLC with the Random achieves an 8.75\% improvement in performance. This significant enhancement demonstrates that even a simple random selection of coresets can benefit substantially from the CLC strategy, which effectively mitigates the quantization errors in intermediate layers. Similarly, when CLC is combined with Forgetting, the performance improves by 6.56\%. The effectiveness of CLC is further underscored by its combination with other coreset selection methods. For instance, integrating CLC with CD results in a 9.92\% performance improvement, while combining it with Moderate method yields an 11.26\% enhancement. These results suggest that CLC can significantly boost the performance of quantized models regardless of the coreset selection method used. When CLC is combined with  ACS, the performance improvement reaches 8.41\%.  This highlights the robustness of CLC in enhancing model performance, as ACS is already an advanced method that selects samples based on their impact on gradients. Overall, the CLC strategy is highly adaptable and can effectively improve the performance of quantized models when combined with various coreset selection methods.

\begin{figure}[t]
   \centering
   \begin{subfigure}[b]{0.235\textwidth}
       \includegraphics[width=\textwidth]{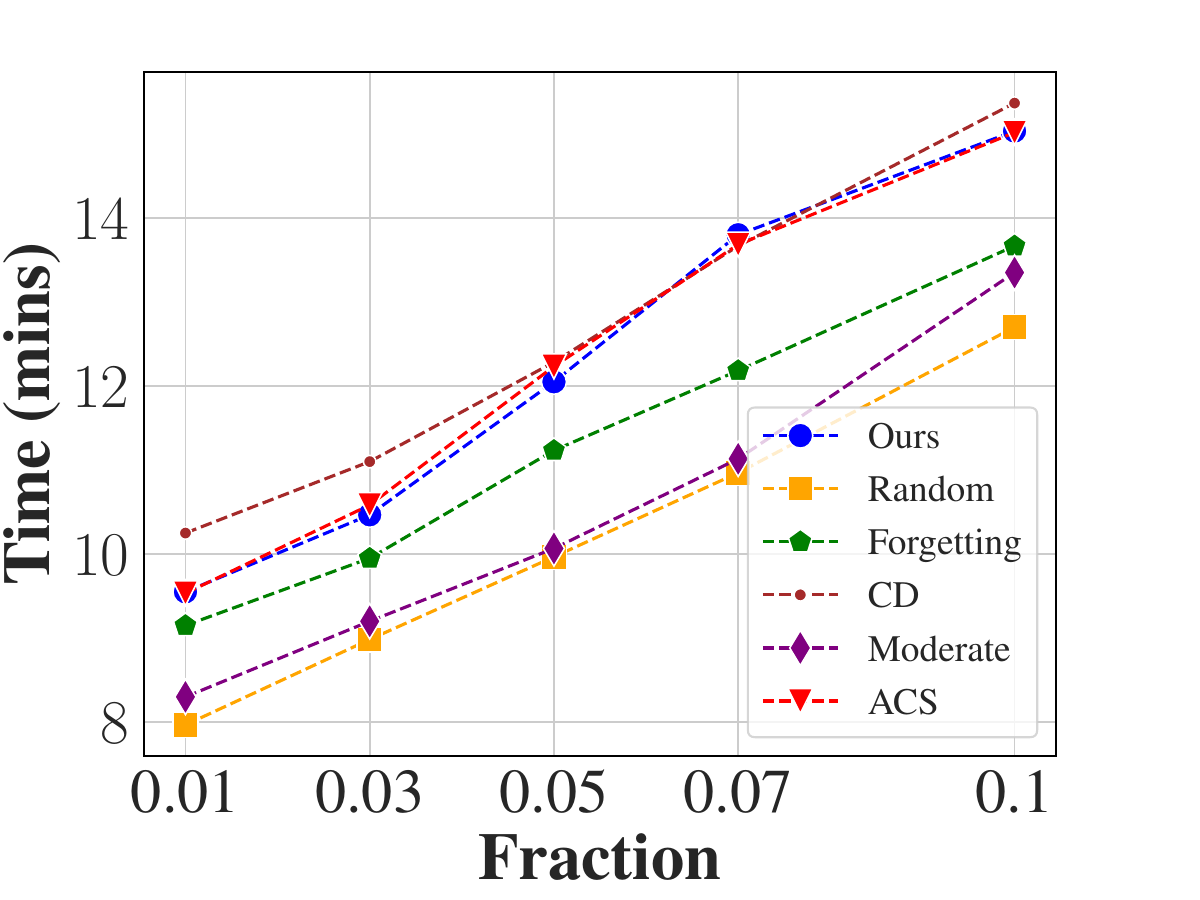}
       \caption{Comparison of Efficiency.}
       \label{fig2:subfig3}
   \end{subfigure}
   \begin{subfigure}[b]{0.235\textwidth}
       \includegraphics[width=\textwidth]{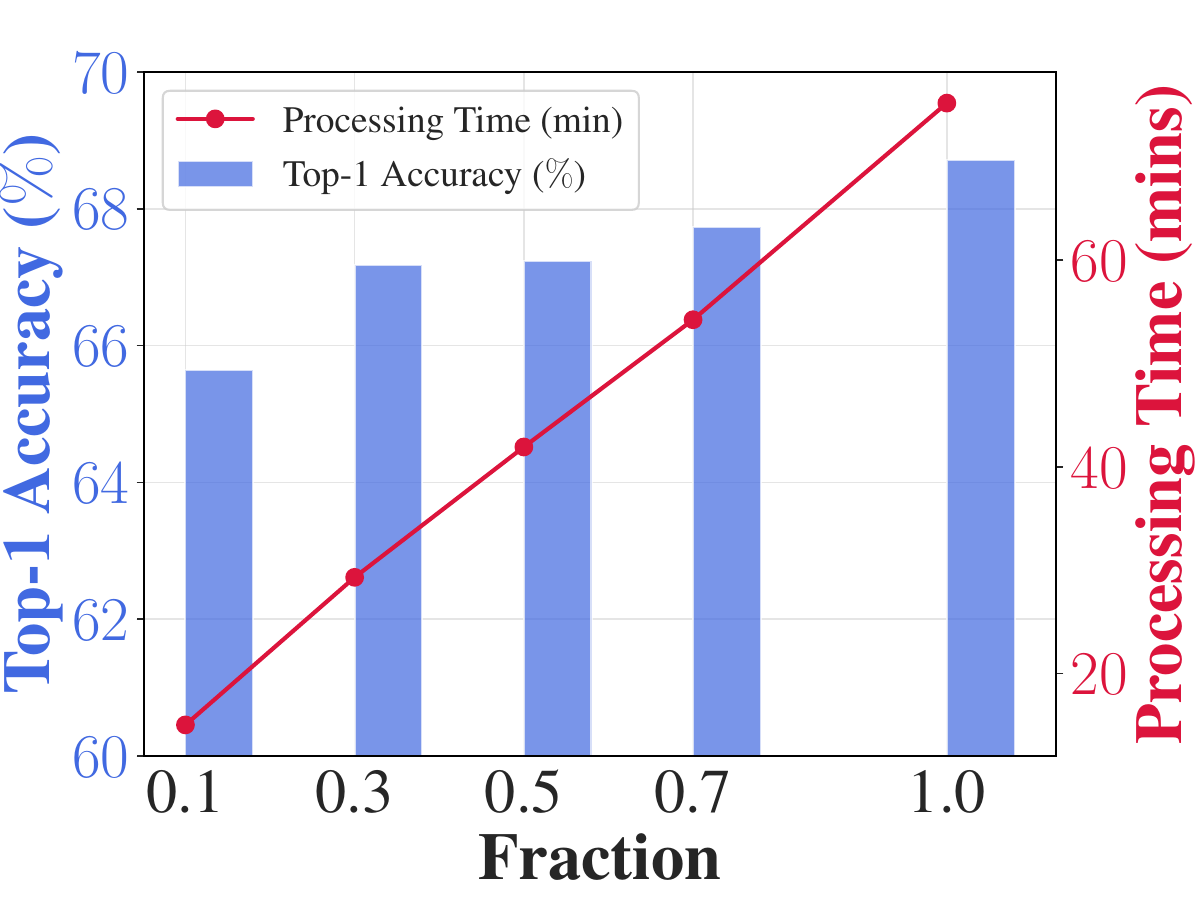}
       \caption{Time-Accuracy Analysis.}
       \label{fig2:subfig4}
    \end{subfigure}
    \caption{(a) The efficiency of different methods.
(b) Balance between efficiency and accuracy of our method. The experimental setup is 2-bit quantized MobileNetV2 on CIFAR-100.}
    \label{fig4}
\end{figure}

\subsubsection{Efficiency Analysis}
We compare the efficiency of different methods, as shown in Fig.~\ref{fig2:subfig3}. Random demonstrates the highest efficiency, taking only $7.97$ minutes to train on $1\%$ of the data. However, despite its significant time-saving advantage, Random suffers from severe performance degradation due to the randomness in selecting the coreset. Specifically, Forgetting takes $9.15$ minutes, Moderate takes $8.30$ minutes, ACS takes $9.53$ minutes, and our method takes $9.55$ minutes. Although our method is slightly slower than ACS, it is faster than CD ($10.25$ minutes). More importantly, our method can significantly improve model performance while maintaining high efficiency. This indicates that our method achieves a good balance between efficiency and performance, avoiding the significant performance drop caused by the Random and Moderate methods' overemphasis on efficiency and overcoming the low efficiency of CD method. Our method provides an effective solution for efficient and high-performance model training on small-scale coresets.

We further analyze the balance between efficiency and accuracy for our method. As shown in Fig.~\ref{fig2:subfig4}, when training with $20\%$ of the data, our method takes $29.32$ minutes and achieves a Top-1 accuracy of $67.18\%$. In contrast, training with the full dataset ($100\%$) takes $75.22$ minutes, resulting in a Top-1 accuracy of $68.72\%$. The time cost increases by $157\%$, but the accuracy only improves by $1.54\%$. This indicates that using a small fraction of the data (e.g., $20\%$) with our method can achieve a significant reduction in training time while maintaining a high level of accuracy. The marginal gain in accuracy from increasing the data fraction beyond $20\%$ is relatively small compared to the substantial increase in training time. Therefore, our method provides an effective trade-off between efficiency and accuracy, especially when computational resources are limited or when faster training is desired without significant loss in model performance.

\begin{table}[t]
\centering
\caption{Results with larger fractions. we report
the Top-1 and Top-5 accuracy of quantized MobileNetV2 on
CIFAR-100. The full-precision MobileNetV2 achieves a Top-1 accuracy of 72.56\% and a Top-5 accuracy of 91.93\% on CIFAR-100. }
\begin{tabular}{c|cc|cc}
\toprule 
Method         & Bit width  &Fraction   & Top-1  & Top-5 \\ \toprule 
ACS           & 2w32a    &$20\%$    &65.03      &89.14      \\
Ours           &2w32a      &$20\%$    &\cellcolor{gray!20}\textbf{65.86}     &\cellcolor{gray!20}\textbf{89.70}      \\
\midrule
ACS           & 2w32a    &$30\%$    &66.32      &89.19      \\
Ours           &2w32a      &$30\%$    &\cellcolor{gray!20}\textbf{67.18}     &\cellcolor{gray!20}\textbf{90.53}      \\
\midrule
ACS           & 2w32a    &$50\%$     &67.13       & 90.28     \\
Ours           &2w32a      &$50\%$     &\cellcolor{gray!20}\textbf{67.24}      &\cellcolor{gray!20}\textbf{90.50}      \\
\midrule
ACS           & 2w32a     &$70\%$     &67.51     &90.06      \\
Ours           &2w32a      &$70\%$     &\cellcolor{gray!20}\textbf{67.73}      &\cellcolor{gray!20}\textbf{91.01}      \\
\midrule
ACS           & 2w32a     &$100\%$     &67.67      &90.31     \\
Ours           &2w32a      &$100\%$     &\cellcolor{gray!20}\textbf{68.72}      &\cellcolor{gray!20}\textbf{91.21}      \\
\midrule
\end{tabular}
\label{tab6}
\end{table}

\subsubsection{Effect of Larger Fractions}
To provide a more comprehensive validation of our method, we present results with larger fractions. As shown in Table \ref{tab6}, our approach continues to demonstrate superior performance compared to the current state-of-the-art method, ACS, even at larger fractions. At smaller data fractions (such as \(20\%\) and \(30\%\)), our method outperforms ACS in terms of both Top-1 and Top-5 accuracy, achieving improvements of \(0.83\%\) and \(0.86\%\) in Top-1 accuracy, and \(0.56\%\) and \(1.34\%\) in Top-5 accuracy, respectively. This demonstrates that our method can more effectively utilize limited samples for training when data is restricted, thereby enhancing model performance. However, as the data fraction increases, the magnitude of performance improvement gradually diminishes. At a \(50\%\) data fraction, the Top-1 accuracy is improved by only \(0.11\%\), and the Top-5 accuracy by \(0.22\%\); while at a \(100\%\) data fraction, the Top-1 accuracy is enhanced by \(1.05\%\), and the Top-5 accuracy by \(0.90\%\). This phenomenon indicates that when the data volume approaches full-data training, the room for model performance improvement is relatively limited. Nevertheless, our method still manages to deliver certain performance gains, which further proves its effectiveness across different data scales.

\begin{table}[t]
\centering
\caption{Results with different teacher models. The student model is the quantized ResNet-18 with both weights and activations quantized to 2-bit. We report
the Top-1 accuracy. }
\begin{tabular}{c|ccc}
\toprule 
Teacher         & ResNet-18  &ResNet-34   & ResNet-101   \\ \toprule 
Top-1(\%)            & 46.24   &46.27    &46.34          \\
\midrule
\end{tabular}
\label{tab7}
\end{table}

\subsubsection{Effect of Different Teacher Models }
For the ResNet-18 model on ImageNet-1K, we demonstrate the impact of different teacher models on the results. As shown in Table \ref{tab7}, more complex teacher models can guide the student model to achieve better performance. Specifically, when using ResNet-101 as the teacher model, the Top-1 accuracy of the quantized ResNet-18 with both weights and activations quantized to 2-bit reaches 46.34\%, which is slightly higher than that achieved with ResNet-34 as the teacher model at 46.27\% and ResNet-18 as the teacher model at 46.24\%. This indicates that during the knowledge distillation process, more complex teacher models can provide more informative guidance for the student model, thereby enhancing the performance of the quantized model.

\begin{figure*}[t]
\centering
\begin{minipage}[t]{0.25\textwidth}
\centering
\includegraphics[width=\textwidth]{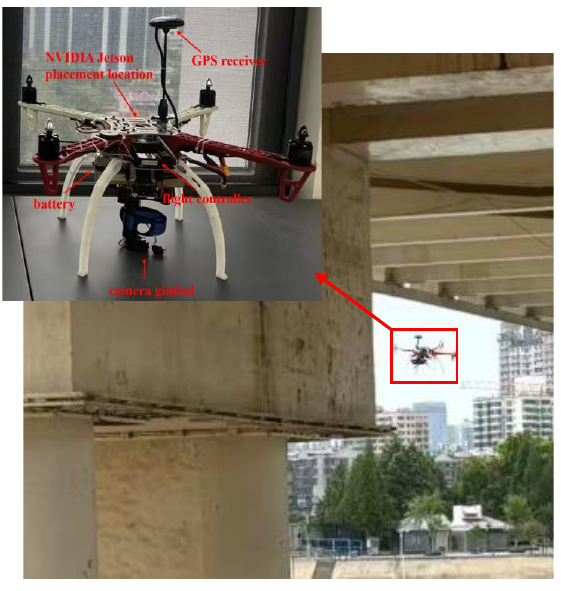} 
\caption{UAV prototype.}
\label{fig5}
\end{minipage}
\begin{minipage}[t]{0.7\textwidth}
\centering
\includegraphics[width=\textwidth]{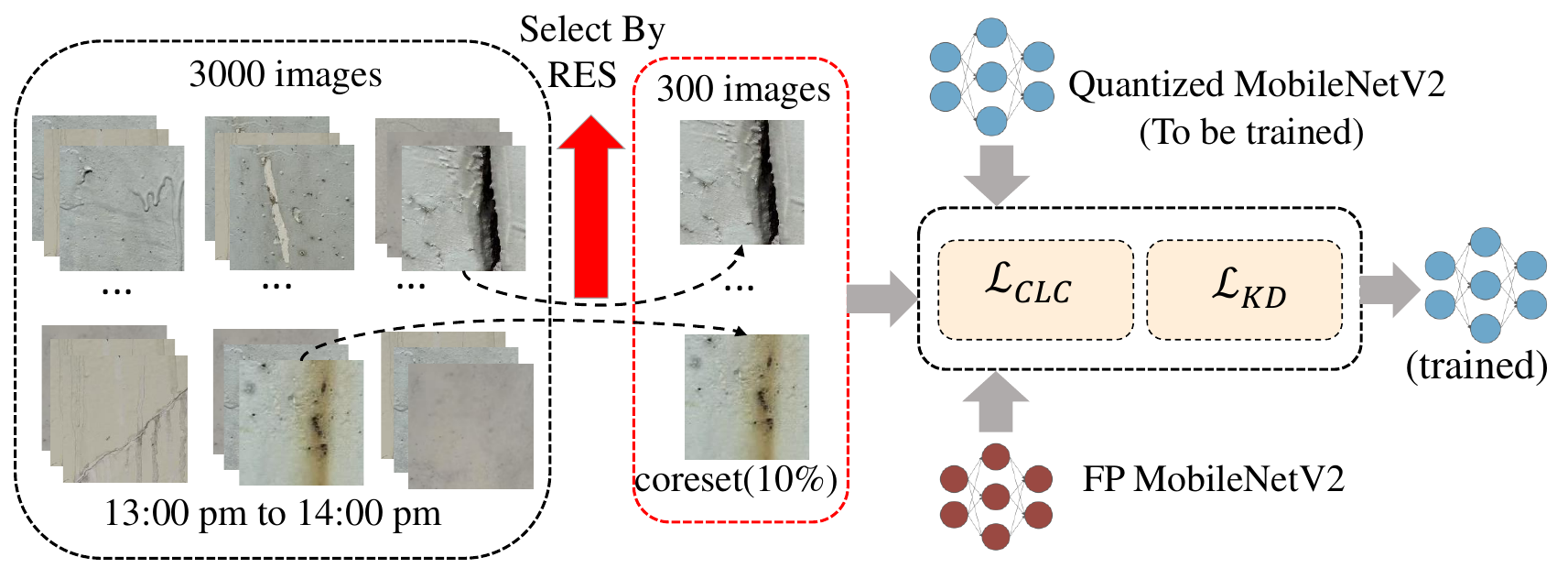} 
\caption{The process of coreset selection and quantized model training.}
\label{fig6}
\end{minipage}
\end{figure*}

\begin{figure}[t]
\centering
\includegraphics[width=0.48\textwidth]{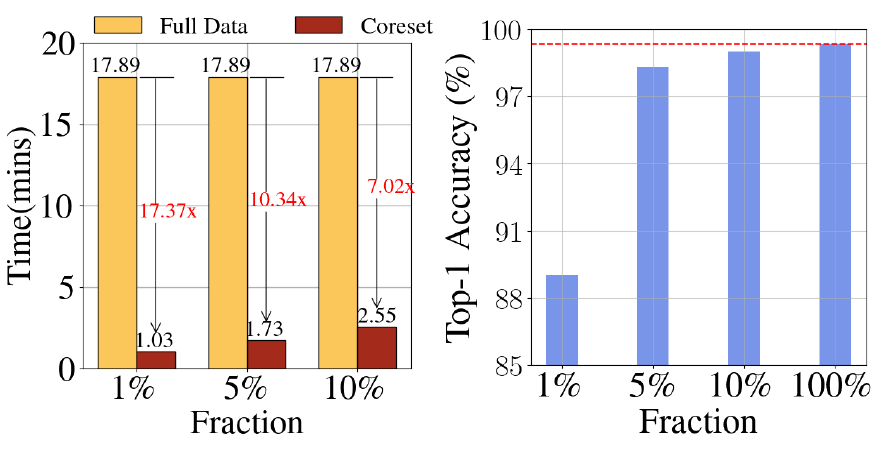} 
\caption{Analysis of Top-1 Accuracy and Training Time.
}
\label{fig7}
\end{figure}

\section{A Case Study}
We evaluate QuaRC on a UAV‐based crack‐detection task for infrastructure inspection. In this setting, UAVs equipped with onboard vision systems must run lightweight quantized models under tight compute and energy budgets. To accommodate bridge‐ and building‐specific variations in crack morphology and lighting—and to avoid transmitting sensitive inspection images to the cloud—quantized models must be retrained locally using newly captured data. However, full‐dataset QAT onboard is infeasible on resource‐constrained hardware. By selecting a small, representative subset (1–10\% of data), QuaRC enables rapid on‐device retraining of quantized models without degrading detection accuracy.

Figure~\ref{fig5} shows our UAV platform, which integrates an NVIDIA Jetson Orin NX 16 GB edge module and a Homer digital link (1.5 km range). We begin by uploading full‐precision MobileNetV2 weights (pretrained on ImageNet‐1K) to the UAV. The UAV then acquires 3,300 concrete‐surface images (3,000 for training, 300 for testing) over a one‐hour flight and annotates cracks offline. As depicted in Figure~\ref{fig6}, we first fine‐tune the full‐precision model for 20 epochs on all 3,000 images, achieving 99.33\% Top-1 accuracy. Next, we run the RES algorithm to select the 300 most informative images (10\%), then retrain the 2-bit quantized MobileNetV2 with our Cascaded Layer Correction strategy. As shown in Figure~\ref{fig7}, this coreset QAT completes in 2.55 minutes—7.02× faster than full‐dataset QAT—while delivering 99.00\% Top-1 accuracy (vs.\ 99.33\% on the full dataset). These results confirm that QuaRC substantially accelerates on-device QAT while preserving model performance.

\section{Related Work and Discussion}

In this section, we first introduce the relevant studies on integrating coreset selection methods with QAT from the aspects of \textit{model quantization} and \textit{coreset selection}. Subsequently, we discuss the limitations and potential extensions of QuaRC.

\subsection{Related Work}
\noindent\textbf{Model quantization:} As one of the key technologies for deep learning model compression, model quantization~\cite{yang2024communication,luo2021binarized,luo2025bi} aims to reduce computational costs and memory usage by converting full-precision floating-point numbers into low-bit integers, thereby enabling models to be deployed on resource-constrained hardware. Quantization methods are categorized into Post-Training Quantization (PTQ) \cite{fang2020post,wang2020towards} and Quantization-Aware Training (QAT) \cite{zhou2016dorefa,esser2020learned}, with QAT showing better performance for extremely low bitwidths (4 bits or fewer). Techniques like Dorefa-net \cite{zhou2016dorefa} focus on accelerating training and inference with low bitwidth parameters, while LSQ \cite{esser2020learned} and LSQ+  \cite{bhalgat2020lsq+} use learnable quantization parameters to achieve better accuracy. Although QAT provides superior performance, it requires retraining on the entire dataset, increasing computational costs. Our work focuses on reducing the computational overhead of QAT so that it can be applied to resource-constrained edge devices.

\noindent\textbf{Coreset selection:} By selecting a subset that can represent the key features of the entire dataset, coreset selection techniques can improve training efficiency, reduce the demand for computational resources, and maintain model performance. Common methods for evaluating sample importance in coreset selection include geometry-based \cite{agarwal2020contextual,xia2023moderate}, decision boundary-based \cite{ducoffe2018adversarial,margatina2021active}, and gradient-based \cite{mirzasoleiman2020coresets,paul2021deep} approaches. These coreset selection methods are designed for full-precision models and do not consider the characteristics of quantized models. Fortunately, ACS \cite{huangrobust} considers the impact of samples on the quantized model from the perspective of gradients, and for the first time, applies coreset selection to QAT. Due to the consideration of quantization characteristics, the performance of ACS is significantly better than traditional methods. Unlike ACS, we approach from the perspective of quantization errors, selecting the samples that best reflect the model's quantization errors as the coreset.

Applying coreset selection to QAT can enhance its training efficiency. However, when using a small-scale coreset for QAT, there can be a notable decline in the performance of the quantized model. Our work aims to \textbf{mitigate this performance degradation of the quantized model when employing a small-scale coreset for QAT.}

\subsection{Discussion}

\noindent\textbf{Limitations:} Despite the significant improvements achieved by QuaRC in QAT on small-scale coresets, there are still some limitations. First, the calculation of the RES requires both the full-precision model and the quantized model to perform inference on all samples simultaneously, which introduces additional computational overhead during the coreset selection phase. Although the overhead is significantly reduced compared to the backpropagation process, it can still be a bottleneck for edge devices with extremely limited resources.  Second, the CLC strategy demands that the full-precision correction model be strictly aligned in structure with the quantized model. If the teacher model and the correction model are not consistent, the training cost will increase. Experiments have shown that although using a deeper teacher model can slightly improve accuracy, the optimal solution for on-device training is to set the teacher model and the correction model as the same model, thereby achieving the best balance between accuracy and efficiency.

\noindent\textbf{Potential Extensions:} While our current work focuses on enhancing quantization-aware training for image classification models, the principles underlying QuaRC—quantization error-aware coreset selection and intermediate layer correction—exhibit promising potential for broader applications. RES selects the coreset by capturing quantization errors, and its nature is task-agnostic, allowing it to be easily adapted to other visual tasks such as object detection\cite{zhao2019object, zou2023object} and semantic segmentation\cite{mo2022review,thisanke2023semantic}. CLC mitigates error accumulation by aligning intermediate features layer-by-layer and is equally applicable to the quantization training of hierarchical network architectures, such as those based on the Transformer framework\cite{devlin2019bert, brown2020language}. However, the current limitations in computational power and memory on edge devices still pose significant challenges for local training of object detection models and large models based on the Transformer architecture. Nevertheless, with the continuous advancement of edge computing hardware, the computational bottlenecks on edge devices are expected to be gradually overcome. Against this backdrop, our future work will focus on the aforementioned directions, aiming to facilitate the implementation of efficient quantization-aware training techniques on edge devices.

\section{Conclusion}
This paper proposes an innovative solution, QuaRC, to address the core challenge of Quantization-Aware Training (QAT) on edge devices: the accumulation of quantization errors when using small-scale coresets. In our research, we delve into two major challenges encountered when applying coreset selection to QAT: 1) the failure to consider quantization errors during coreset selection, and 2) the difficulty in eliminating intermediate layer errors in the quantized model during the training phase. To tackle these challenges, we introduce the Relative Entropy Score as the selection metric in the coreset selection stage and employ the Cascaded Layer Correction strategy during the training phase of the quantized model to eliminate errors in the intermediate layers. Extensive experiments demonstrate the significant advantages of our proposed method over existing approaches. Our work effectively mitigates the performance degradation of quantized models when using small-scale coresets, thereby making it possible to perform QAT directly on edge devices.

\bibliographystyle{IEEEtran}
\bibliography{main}

\vfill

\end{document}